\documentclass[letterpaper, 10 pt, conference]{ieeeconf}  
\IEEEoverridecommandlockouts             
\makeatletter
\let\NAT@parse\undefined
\makeatother
\usepackage{hyperref}
\usepackage{xurl}
\usepackage{algorithm}
\usepackage{algorithmic}
\usepackage[linesnumbered,ruled,vlined,algo2e]{algorithm2e}
\usepackage{amsfonts}
\usepackage{amsmath}
\usepackage{amssymb}
\usepackage[ansinew]{inputenc} 
\usepackage{xcolor}
\usepackage{mathtools}
\usepackage{graphicx}
\usepackage{caption}
\usepackage{subcaption}
\usepackage{import}
\usepackage{multirow}
\usepackage{cite}
\usepackage[export]{adjustbox}
\usepackage{breqn}
\usepackage{mathrsfs}
\usepackage{acronym}
\usepackage{setspace}
\usepackage{bm}
\usepackage{stackengine}
\usepackage{stackengine}
\usepackage{needspace}
\usepackage{comment}
\usepackage{siunitx}
\usepackage{lipsum}

\usepackage{amsthm}
\usepackage{amssymb}
\usepackage{svg}
\usepackage{lipsum}
\usepackage{wasysym}
\usepackage{mathrsfs}
\usepackage[T1]{fontenc}

\theoremstyle{plain}
\newtheorem{theorem}{Theorem}
\newtheorem{lemma}{Lemma}

\theoremstyle{definition}
\newtheorem{assumption}{Assumption}
\newtheorem{definition}{Definition}
\newtheorem{problem}{Problem}
\newtheorem{procedure2}{Procedure}

\theoremstyle{remark}
\newtheorem{remark}{Remark}

\title{ \LARGE \bf
  Efficient Coordination and Synchronization of Multi-Robot Systems Under Recurring Linear Temporal Logic
}

\author{Davide Peron$^1$, Victor Nan Fernandez-Ayala$^2$, Eleftherios E. Vlahakis$^2$, and Dimos V. Dimarogonas$^2$
\thanks{This work was supported by the ERC CoG LEAFHOUND, the EU
CANOPIES project, the Knut and Alice Wallenberg Foundation (KAW) and the Digital Futures Smart Construction project.}
\thanks{$^1$ Department of Information Engineering, University of Padova, 35122, Padova, Italy. Email: {\tt\small davide.peron.3@studenti.unipd.it} 
$^2$ Division of Decision
and Control Systems, School of Electrical Engineering and Computer
Science, KTH Royal Institute of Technology, 10044, Stockholm, Sweden. Emails: \tt\small\{vnfa,vlahakis,dimos\}@kth.se}
}

\begin{document}
\newcommand{\until}[0]{\mathsf{U}}

\def\triangleq{\mathrel{\ensurestackMath{\stackon[1pt]{=}{\scriptstyle\Delta}}}}

\maketitle
\thispagestyle{empty}
\pagestyle{empty}

%%%%%%%%%%%%%%%%%%%%%%%%%%%%%%%%%%%%%%%%%%%%%%%%%%%%%%%%%%%%%%%%%%%%%%%%%%%%%%%%
% !TEX root = template.tex

\begin{abstract}
We consider multi-robot systems under recurring tasks formalized as linear temporal logic (LTL) specifications. To solve the planning problem efficiently, we propose a bottom-up approach combining offline plan synthesis with online coordination, dynamically adjusting plans via real-time communication. To address action delays, we introduce a synchronization mechanism ensuring coordinated task execution, leading to a multi-agent coordination and synchronization framework that is adaptable to a wide range of multi-robot applications. The software package is developed in Python and ROS2 for broad deployment. We validate our findings through lab experiments involving nine robots showing enhanced adaptability compared to previous methods. Additionally, we conduct simulations with up to ninety agents to demonstrate the reduced computational complexity and the scalability features of our work.
\end{abstract}

% !TEX root = template.tex

\section{Introduction} \label{sec:introduction}
Successful coordination and synchronization in multi-agent systems (MAS) is essential for completing complex, interdependent tasks that individual agents cannot handle alone,  making it one of the key challenges in robotic applications such as in logistics \cite{logistics} and precision agriculture \cite{agriculture}. As the number of agents and the complexity of tasks increase, ensuring scalability and conflict-free operation becomes more challenging, especially when tasks are expressed through temporal logic constraints. Therefore, finding an efficient and scalable solution is necessary to manage computational complexity and maintain robust performances. 
\par Linear Temporal Logic (LTL) is a logical formalism suitable for defining linear-time properties, widely used in formal verification, particularly in computer systems, and increasingly in robotics for specifying tasks in MAS \cite{ltl_planner_3}, \cite{multiagentltl1}. The introduction of Product B\"{u}chi Automata (PBA) has significantly advanced LTL compliance in MAS by systematically generating control strategies \cite{model-checking}. However, scalability remains a challenge as task complexity and the number of agents increase. Many recent approaches address this computational limitation by avoiding the direct construction of the PBA and relying on alternative control techniques, such as the sampling-based synthesis in \cite{large-scale, stylus} or reinforcement learning (RL) techniques in \cite{RL}. Although efficient, these techniques introduce a stochastic element to planning, whereas we aim to maintain the deterministic guarantees of graph search methods \cite{model-checking}. Other approaches focus on task assignment to individual agents to limit the exponential growth of the state space, rather than using centralized planning \cite{decentralized}. These methods often target specific robotic functions, such as motion \cite{motion2}, which hinders their applicability to more complex tasks, given the specificity of their context. Recent efforts, such as \cite{scratches}, introduce a task grammar that enables flexible, high-level task specification for heterogeneous agents, alleviating the need for explicit task assignments while maintaining correctness guarantees. Furthermore, synchronization mechanisms are becoming increasingly critical to ensure coordinated actions in MAS \cite{synchronization}.
\par In \cite{meng_paper} local tasks defined by syntactically cosafe LTL (sc-LTL) formulas \cite{scLTL} are considered, with agents navigating a grid-structured workspace and performing local, collaborative, and assistive actions. An offline planner generates initial plans for each agent, which are then adapted through a Request, Reply, and Confirmation cycle. The latter utilizes the PBA to incorporate assistive actions into the agents' plans. While robust, this approach is limited to finite-time tasks defined by sc-LTL formulas. In this work, we adapt \cite{meng_paper} by addressing its limitations and enhancing computational efficiency. Unlike the earlier study, which relied solely on simulations and did not fully capture edge cases, our work integrates both experiments and more exhaustive simulations. Specifically, we introduce a region of interest (ROI) based representation to reduce the size of the finite transition system (FTS) \cite{model-checking} (Sec. \ref{subsec:motion-ts}), and subsequently define a new subclass of LTL called recurring LTL to specify tasks that repeat infinitely often (Sec. \ref{subsec:task}). In Alg. \ref{alg:reply}  we eliminate the use of the PBA for plan adaptation in favor of the simpler FTS. Additionally, we reduce the complexity of selecting collaborative agents by using filtering Proc. \ref{proc:filtering} to warm-start the underlying mixed integer program (MIP) \cite{mip}. Lastly, we present a robust synchronization mechanism within the ROS2 \cite{ros2} framework that further enables effective collaboration among agents (Sec \ref{subsec:res-synchro}). Our approach is computationally efficient for real-world usage as proven with a team composed of nine commercially available robotic platforms (Sec. \ref{subsec:exp-results}) and suitable for large-scale robotic deployments as shown by a simulation with ninety robots (see Sec. \ref{subsec:exp-scalability}). The code is available at \cite{repo}.

% !TEX root = template.tex

\section{Preliminaries} \label{sec:preliminaries}

\subsection{LTL and B\"{u}chi Automaton} \label{subsec:LTL}
Atomic propositions are Boolean variables that can either be true or false. An LTL formula is defined over a set of atomic propositions ($\Psi$), the Boolean connectors: negation ($\neg$), conjunction ($\wedge$), and the temporal operators, \textit{next} ($\Circle$) and \textit{until} ($\until$). It is specified according to the following syntax \cite{model-checking}:
$ \varphi::= \top\mid a\mid\neg\varphi\mid\varphi_1\wedge\varphi_2\mid\Circle\varphi\mid\varphi_1\until\varphi_2$ where $a\in\Psi$, and $\top\triangleq True$. Operators \textit{always} ($\square$), \textit{eventually} ($\lozenge$) can be derived from the syntax above \cite[Ch. 5]{model-checking}. The satisfaction of an LTL formula $\varphi$ is achieved over words and the language of such words can be captured through a nondeterministic B\"{u}chi automaton (NBA) \cite{buchi_book}, defined as $\mathcal{B}=\left(S, S_0,  2^{\Psi}, \delta, \mathcal{F}\right)$, where $S$ is a finite set of states, $S_{0} \subseteq S$ is the set of initial states, $2^{\Psi}$ is the set of all alphabets, $\delta: S \times 2^{\Psi} \rightarrow 2^{S}$ is the transition function, $\mathcal{F} \subseteq S$ is the set of accepting states.

\subsection{LTL Planning}\label{subsec:planning}
We briefly highlight the key points of the initial planning (implemented offline), which was developed in \cite{ltl_planner}. The planner node takes as input an LTL
task $\varphi$, and an FTS $\mathcal{T}_w$. First, $\varphi$ is used to generate the NBA $\mathcal{B}_{\varphi}$ via the LTL2BA software \cite{LTL2BA}. Building upon the work of \cite{ltl_planner_2}, a PBA $\mathcal{A_P}$ \cite{ltl_planner} is built, then through model checking techniques \cite{ltl_planner_3} the optimal run that satisfies the given LTL task and the corresponding sequence of actions are found. For all the details we refer to \cite{ltl_planner}, \cite{ltl_planner_2}, \cite{ltl_planner_3}, \cite{meng_paper}.

% !TEX root = template.tex

\section{System Setup} \label{sec:problem}
Let a group of heterogeneous agents $\mathcal{N} = \{a_i \mid i=1,2,...,N\}$ operate within a partially known workspace. Each agent can execute primitive actions that may require assistance from others. The agents are connected via a shared network. Within this framework, agents can directly exchange messages with any other agent in the workspace. %Next, we define the agents' models and their assigned tasks.
\subsection{System description}
\subsubsection{Motion Transition System} \label{subsec:motion-ts}
Agent $a_i$'s motion within the workspace is modeled as an FTS. Our approach focuses on a set of ROIs while a low-level controller handles obstacle avoidance and inter-region movement. This approach significantly reduces computational complexity compared to a fully partitioned workspace but sacrifices the ability to track, at any time, the agent's exact location within the FTS. Each agent $a_i$ is aware only of the set of $M^{a_i}$ ROIs, denoted by $\Pi^{a_i}_{\mathcal{M}}=\{\pi^{a_i}_1,\pi^{a_i}_2,...,\pi^{a_i}_{M^{a_i}}\}$. The FTS assigned agent $a_i$ is:
\begin{equation} \label{eq:motion-fts} \mathcal{T}^{a_i}_{\mathcal{M}}\triangleq\left(\Pi^{a_i}_{\mathcal{M}}, \Pi^{a_i}_{\mathcal{M},0}, \Psi^{a_i}_{\mathcal{M}}, \Sigma^{a_i}_{\mathcal{M}}, \longrightarrow^{a_i}_{\mathcal{M}}, \mathrm{L}^{a_i}_{\mathcal{M}}, \mathrm{T}^{a_i}_{\mathcal{M}}\right),  
\end{equation}
where $\Pi^{a_i}_{\mathcal{M},0}\in\Pi^{a_i}_{\mathcal{M}}$ is the initial ROI, $\Psi^{a_i}_{\mathcal{M}}$ is the set of atomic propositions describing the properties of the workspace, $\Sigma^{a_i}_{\mathcal{M}}$ is the set of movement actions, $\longrightarrow^{a_i}_{\mathcal{M}}\subseteq \Pi^{a_i}_{\mathcal{M}}\times\Sigma^{a_i}_{\mathcal{M}}\times\Pi^{a_i}_{\mathcal{M}}$ is the transition relation, $\mathrm{L}^{a_i}_{\mathcal{M}}:\Pi^{a_i}\rightarrow2^{\Psi^{a_i}_{\mathcal{M}}}$ is the labeling function, indicating the properties held by each ROI, and $\mathrm{T} ^{a_i}_{\mathcal{M}}:\longrightarrow^{a_i}_{\mathcal{M}}\rightarrow\mathbb{R}^+$ is the transition time function, representing the estimated time necessary for each transition.

\subsubsection{Action Model}\label{subsec:action-model}
In addition to its movement actions, agent $a_i$ can perform actions  $\Sigma^{a_i}_{\mathscr{A}} \triangleq \Sigma^{a_i}_l \cup \Sigma^{a_i}_c \cup \Sigma^{a_i}_h$, where $\Sigma^{a_i}_l$ are \textit{local} actions performed independently, $\Sigma^{a_i}_c$ are \textit{collaborative} actions requiring assistance from other agents, and $\Sigma^{a_i}_h$ are \textit{assisting} actions carried out to help others. Lastly, $\sigma_0 = \mathit{None}\in \Sigma^{a_i}_l$ indicates that $a_i$ remains idle. The action model for agent $a_i$ is defined as the tuple:
\begin{equation}\label{eq:action-model}
\mathscr{A}^{a_i} \triangleq \left(\Sigma^{a_i}_{\mathscr{A}}, \Psi^{a_i}_{\mathscr{A}}, \mathrm{L}^{a_i}_{\mathscr{A}}, \mathrm{Cond}^{a_i}, \mathrm{Dura}^{a_i}, \mathrm{Depd}^{a_i}\right),
\end{equation}
where $\Psi^{a_i}_{\mathscr{A}}$ is the set of atomic propositions, $\mathrm{L}^{a_i}_{\mathscr{A}}$ is the labeling function as in \cite{meng_paper}, $\mathrm{Cond}^{a_i}$ is the region properties required to execute an action, $\mathrm{Dura}^{a_i}$ is the action duration, with $\mathrm{Dura}^{a_i}(\sigma_s) = T_s > 0$, and 
$\mathrm{Depd}^{a_i}: \Sigma^{a_i}_{\mathscr{A}} \rightarrow 2^{\Sigma^{\sim a_i}_h} \times 2^{\Pi^{\mathcal{N}}}$ denotes the dependence function, where  $\Sigma^{\sim a_i}_h$ is the set of \textit{external} assisting actions that agent $a_i$ depends on, and $\Pi^{\mathcal{N}}=\cup_{a_i\in\mathcal{N}}\Pi^{a_i}_{\mathcal{M}}$. Given $\sigma_c\in \Sigma^{a_i}_c $, we define the set of actions involved in a collaboration as:
\begin{equation}\label{eq:collaboration}
    \mathcal{C}(\sigma_c)=\{\sigma_c\}\cup\mathrm{Depd}^{a_i}(\sigma_c).
\end{equation}
\begin{definition}\label{def:succesful-collab}
    A collaboration is considered successful if all actions involved are synchronized; i.e. to complete $\sigma_c \in \Sigma_c^{a_i}$, it is necessary that all actions in $\mathcal{C}(\sigma_c)$ start simultaneously.
\end{definition}
\subsubsection{Agent Transition System} \label{subsec:agent-ts}

The planner in Sec. \ref{subsec:planning} requires to define agent $a_i$'s FTS by combining \eqref{eq:motion-fts} and \eqref{eq:action-model}.
\begin{definition}
 Given $\mathcal{T}^{a_i}_{\mathcal{M}}$ and $\mathscr{A}^{a_i}$, a valid FTS for agent $a_i$, according to \cite{model-checking}, can be constructed as follows:
    \begin{equation}\label{eq:agent-model}    \mathcal{T}^{a_i}_{\mathcal{G}}\triangleq\left(\Pi^{a_i}_{\mathcal{G}}, \Pi^{a_i}_{\mathcal{G},0}, \Psi^{a_i}_{\mathcal{G}}, \Sigma^{a_i}_{\mathcal{G}}, \longrightarrow^{a_i}_{\mathcal{G}}, \mathrm{L}^{a_i}_{\mathcal{G}}, \mathrm{T}^{a_i}_{\mathcal{G}}\right), 
    \end{equation}
where $\Pi^{a_i}_{\mathcal{G}} = \Pi^{a_i}_{\mathcal{M}}\times \Sigma^{a_i}_{\mathscr{A}}$ is the set states,
$\Pi^{a_i}_{\mathcal{G},0}=\langle\Pi^{a_i}_{\mathcal{M},0} , \mathit{None}\rangle$ is the initial state,
$\Psi^{a_i}_{\mathcal{G}}$ is the set of atomic propositions,
$\Sigma^{a_i}_{\mathcal{G}}=\Sigma^{a_i}_{\mathcal{M}}\bigcup\Sigma^{a_i}_{\mathscr{A}}$, with $\Sigma^{a_i}_{\mathcal{G}, l}=\Sigma^{a_i}_{\mathcal{M}}\bigcup\Sigma^{a_i}_l$, 
$\longrightarrow^{a_i}_{\mathcal{G}}$ is the transition relation,
$\mathrm{L}^{a_i}_{\mathcal{G}}$ is the labeling function, and 
$\mathrm{T}^{a_i}_{\mathcal{G}}$ is the transition estimated duration \cite{meng_paper}.
\end{definition}
As in \cite{meng_paper},  the path is denoted by  $\tau^{a_{i}}=\pi^{a_i}_{\mathcal{G}, 0} \pi^{a_i}_{\mathcal{G}, 1} \ldots$ its trace by $\mathit{trace}(\tau^{a_{i}})=L_{\mathcal{G}}^{a_{i}}(\pi^{a_i}_{\mathcal{G}, 0}) L_{\mathcal{G}}^{a_{i}}(\pi^{a_i}_{\mathcal{G}, 1}) \ldots$ and, the associated sequence of actions by $\rho^{a_i}=\sigma^{a_i}_0, \sigma^{a_i}_1,\ldots$, i.e., the actions that allow transition between the states of $\tau^{a_{i}}$. 

\subsubsection{Task Specification}\label{subsec:task}
%In \cite{meng_paper} sc-LTL was considered but o
Our focus is on implementing recurring tasks i.e., tasks that repeat infinitely often. We will consider the following syntax $\varphi' ::=\top \mid a \mid \neg a\mid \varphi'_1\wedge\varphi'_2 \mid \lozenge\varphi'$, and for agent $a_i$ we define the recurring task as
\begin{equation}\label{eq:recurringLTL}
\varphi^{a_i}_r=\varphi'_1\wedge\square\lozenge\varphi'_2.
\end{equation}
Note that $\varphi'_2$ cannot start with $\lozenge$ to guarantee the validity of the LTL formula.
Given any satisfying word of $\varphi^{a_i}_r$, inserting a detour i.e., a finite sequence of states, between two consecutive states results in a satisfying word.

%\subsection{Problem Statement}
%We can summarize the problem as follows:
\begin{problem}\label{problem:task}
Given $\mathcal{T}_{\mathcal{G}}^{a_i}$ and the locally assigned task $\varphi^{a_i}_r$, design a distributed coordination and synchronization scheme such that $\varphi^{a_i}_r$ is satisfied for all $a_i \in \mathcal{N}$.%, and 2)   %The algorithm must also compensate for delays induced by the experimental scenario and 
 %all joint actions involved in a collaboration in  \eqref{eq:collaboration} start simultaneously.
\end{problem}
%\begin{remark}
   %Synchronizing actions that require precise timing, such as loading boxes, is critical for successful collaboration. Otherwise, timing discrepancies could lead to failure. 
%\end{remark}

% !TEX root = template.tex

\section{Main Results}\label{sec:results}
% In this section, we present the algorithm's structure and our theoretical results. 
Our approach consists of an offline initial planning and an online adaptation phase. The latter involves agents exchanging request, reply, and confirmation messages to identify the best candidates for assisting with collaborative actions.
\subsection{Offline Initial Planning}\label{subsec:res-planning}
Given the FTS of an agent $\mathcal{T}_{\mathcal{G}}^{a_{i}}$ and the locally assigned task $\varphi^{a_{i}}_r$, the offline planner (see Sec. \ref{subsec:planning}) outputs the optimal initial plan for the agent, $\tau^{a_i}_{init}$, namely a sequence of states that satisfy $\varphi^{a_i}_r$, and the associated sequence of actions $\rho^{a_i}_{init}$. The resulting plan guarantees that $\mathit{trace}\left(\tau^{a_{i}}_{init}\right) \models \varphi^{a_{i}}_r$ where the satisfaction relation is defined in Sec. \ref{subsec:LTL}.

\subsection{Collaboration Request}\label{subsec:res-request}
 Given a collaborative action $\sigma^{a_i}_c\in\Sigma^{a_i}_c$, the request message sent by agent $a_i$ to all the other agents is
    \begin{equation*}
    \mathbf{Req}^{a_i}=\{(\sigma_d, \pi_{d}, T^{a_i}_c) \forall \sigma_d\in\mathrm{Depd}^{a_i}(\sigma^{a_i}_c)\}.
    \end{equation*}
Here, $\sigma_d$ are the assistive actions required to complete $\sigma^{a_i}_c$, $\pi_{d}$ are the regions where $\sigma_d$ takes place, $T^{a_i}_c$ is the time required before $\sigma^{a_i}_c$ starts according to $a_i$'s plan.
%\vspace{-0.3cm}
\begin{algorithm2e}[t]
\DontPrintSemicolon
\SetKwFunction{ChooseROI}{Choose\_ROI}
\SetKwInOut{Input}{Input}
\SetKwInOut{Output}{Output}
\ResetInOut{Output}

\Input{$l$, $\rho^{a_i}$, $H^{a_i}$, $T_{rem}$}
\Output{$\mathbf{Req}^{a_i}$}
\BlankLine
$s=1$, $T^{a_i}_c=T_{rem}$\;
\While{$T^{a_i}_c<H^{a_i}$}
{   
    \If{$\rho^{a_i}[l+s]\in\Sigma^{a_i}_c$}
    {
        \ForAll{$\sigma_d\in\mathrm{Depd}^{a_i}(\rho^{a_i}[l+s])$}
        {
            $\pi_d=$\ChooseROI{}\;
            add $(\sigma_d,\ \pi_d, T^{a_i}_c)$ to $\mathbf{Req}^{a_i}$\;            
        } 
        \Return{$\mathbf{Req}^{a_i} $}
    }
    $T^{a_i}_c = T^{a_i}_c + T^{a_i}_{\mathcal{G}}(\rho^{a_i}[l+s])$, $s=s+1$\;
}
\Return{$\emptyset$}
\caption{Check in horizon and Request}\label{alg:request}
\end{algorithm2e}
%\vspace{-0.3cm}
Algorithm \ref{alg:request} generates $\mathbf{Req}^{a_i}$ by assuming a generic time instant where agent $a_i$ is in state $\pi^{a_i}_{\mathcal{G}, l}$, the $l$-th element of its plan $\tau^{a_i}$ (possibly different from $\tau^{a_i}_{init}$ due to modification from the application of our approach). The agent executes action $\sigma^{a_i}_l$, corresponding to $\rho^{a_i}[l]$. The algorithm inputs the current state and action index ($l$), the action sequence $\rho^{a_i}$, the time remaining for the current action ($T_{rem}$), and the horizon length ($H^{a_i}$), considered to be of a similar order of magnitude of the actions in the experimental scenario. The algorithm checks future actions within the horizon for any collaborative tasks. If found, it creates a request, otherwise, it returns an empty set, sending no message.  
The \ChooseROI function selects an ROI based on the specific experimental setup. As a result, we cannot provide a general implementation; however, the specific one used in our experiment is detailed in Sec. \ref{subsec:exp-system}. Lastly, at the end of each iteration, we update $T^{a_i}_c$ by adding the duration of the last action.
\subsection{Reply}\label{subsec:res-reply}
Given a request for collaboration, $\mathbf{Req}^{a_i}$, the reply message sent by agent $a_j$ to $a_i$ is
\begin{equation*}
        \mathbf{Reply}^{a_j}=\{(\sigma_d, \pi_d, b^{a_j}_{d}, t^{a_j}_{d}) \forall (\sigma_d, \pi_d, T^{a_i}_c) \in \mathbf{Req}^{a_i}\}.        
\end{equation*}
Here, $b^{a_j}_{d}$ is a Boolean variable indicating that agent $a_j$ can execute $\sigma_d$ at region $ \pi_d$ and at time $t^{a_j}_{d}$.

%\vspace{-0.3cm}
\begin{algorithm2e}[t]
\DontPrintSemicolon
%\SetKwFunction{GetDetour}{GetDetour}
\SetKwFunction{Algorithmt}{Alogithm3}
\SetKwFunction{Dijkstra}{Dijkstra}
\SetKwInOut{Input}{Input}
\SetKwInOut{Output}{Output}
\ResetInOut{Output}

\Input{$\mathbf{Req}^{a_i}$, $m$, $\tau^{a_j}$, $\mathcal{T}_{\mathcal{G}}^{a_{j}}$, $\overline{T}^{a_j}$, $T_{rem}$}
\Output{$\mathbf{Reply}^{a_j}$, $D^{a_j}$}
\BlankLine
\ForAll{$(\sigma_d, \pi_d, T^{a_i}_c)\in \mathbf{Req}^{a_i}$}
{
    \eIf{$\overline{T}^{a_j}$ and $\langle \pi_d , \sigma_d\rangle \in \Pi^{a_j}_{\mathcal{G}}$}{
        $\pi^{a_j}_{\mathcal{G}, init}=\tau^{a_j}[m+1]$, \quad $\pi^{a_j}_{\mathcal{G}, fin}=\tau^{a_j}[m+2]$\;
        $\pi^{a_j}_{\mathcal{G}, targ}=\langle\sigma_d, \pi_d\rangle$\; 

        $(D_1, C_1)$=\Dijkstra{$\mathcal{T}_{\mathcal{G}}^{a_{j}}$, $\pi^{a_j}_{\mathcal{G}, init}$, $\pi^{a_j}_{\mathcal{G}, targ}$ }\;
        $(D_2, C_2)$=\Dijkstra{$\mathcal{T}_{\mathcal{G}}^{a_{j}}$, $\pi^{a_j}_{\mathcal{G}, targ}$, $\pi^{a_j}_{\mathcal{G}, fin}$ }\;
        $D^{a_j}(\sigma_d)=D_1+D_2[1:end]$\;
        $t^{a_j}_{d}=\sum_{i=0}^{len(C_1)-2} C_1[i]$\;
        add $(\sigma_d, \pi_d, \top ,  T_{rem}+t^{a_j}_{d})$ to $\mathbf{Reply}^{a_j}$\;
    }{
      add $(\sigma_d, \pi_d, \bot, K)$ to $\mathbf{Reply}^{a_j}$\;
    }
}
\Return{$D^{a_j},\ \mathbf{Reply}^{a_j}$}
\caption{Reply of  $a_j$ to a request from  $a_i$}\label{alg:reply}
\end{algorithm2e}
\begin{comment}

\begin{algorithm2e}
\DontPrintSemicolon
\SetKwInOut{Input}{Input}
\SetKwInOut{Output}{Output}
\ResetInOut{Output}
\Input{$\pi_{\mathcal{G}, init}$, $\pi_{\mathcal{G}, fin}$, $\pi_{\mathcal{G}, targ}$, $\mathcal{T}_{\mathcal{G}}$ }
\Output{$T_D$, $D$}
\BlankLine
$(D_1, C_1)$=\Dijkstra{$\mathcal{T}_{\mathcal{G}}$, $\pi_{\mathcal{G}, init}$, $\pi_{\mathcal{G}, targ}$ }\;
$(D_2, C_2)$=\Dijkstra{$\mathcal{T}_{\mathcal{G}}$, $\pi_{\mathcal{G}, targ}$, $\pi_{\mathcal{G}, fin}$ }\;
$D$=$D_1$+$D_2[1:end]$

$T_D$=$\sum\limits_{i=0}^{len(C_1)-2} C_1[i]$\

\Return{$D,\ T_D$}
\caption{Build a detour from the given states}\label{alg:shortest_path}
\end{algorithm2e}
\end{comment}
%\vspace{-0.3cm}
Algorithm \ref{alg:reply} generates $\mathbf{Reply}^{a_j}$ by assuming a request arrives when agent $a_j$ is in state $\pi^{a_j}_{\mathcal{G}, m}$, the $m$-th element of its plan $\tau^{a_j}$, currently executing action $\sigma^{a_j}_m$. Alg. \ref{alg:reply} takes as inputs the request from $a_i$, the current state and action index ($m$), the plan $\tau^{a_j}$, the FTS $\mathcal{T}_{\mathcal{G}}^{a_j}$, the availability indicator $\overline{T}^{a_j}$ ($\top$ if available, $\bot$ if collaborating), and the time remaining for the current action $T_{rem} = \max(T^{a_j}_{\mathcal{G}}(\sigma^{a_j}_m) - \Delta t, 0)$ with $\Delta t$ the elapsed time from the start of $\sigma^{a_j}_m$. It outputs $\mathbf{Reply}^{a_j}$ and a detour dictionary $D^{a_j}$, which includes possible path detours, e.g., $D^{a_j}(\sigma_d)$ represents a detour that includes the state $\langle\pi_d, \sigma_d\rangle$. The algorithm checks if $a_j$ can assist with any requested actions, if so, from lines 3 to 7, it builds the detour between initial ($\pi^{a_j}_{\mathcal{G}, init}$) and final ($\pi^{a_j}_{\mathcal{G}, fin}$) states, passing through a target ($\pi^{a_j}_{\mathcal{G}, targ}$) state. It uses a modified \Dijkstra algorithm that returns the shortest path $D$ and an array $C$ with the associated action costs. The first call yields $D_1$ and $C_1$ i.e., the path and cost from the initial to the target state, similarly the second call produces $D_2$ and $C_2$ from the target to the final state. Then it merges the two paths into the detour saving it in the dictionary $D^{a_j}$. Lastly, it calculates the time $t^{a_j}_{d}$ for $a_j$ to start the assistive action, excluding the current action's completion time. Note that the summation avoids adding the cost of $\sigma_d$, stopping at $len(C_1)-2$. If a state is not feasible or $a_j$ is occupied, it sets $t^{a_j}_d = K$, where $K \gg T^{a_i}_c$; in practice, we set $K=10 T^{a_i}_c$, but any finite value would work. 
\begin{remark} \label{rem:delays}
    Note that $t^{a_j}_d$, $T^{a_i}_c$ are nominal times, but an agent might be affected by finite delays due to actions taking longer than expected in experiments.
\end{remark}
The definition of $T_{rem} \geq 0$ ensures nonnegative times from action delays. Note $a_i$ replies negatively to its request.
\subsection{Confirmation}\label{subsec:res-confirm}
 Given $\mathbf{Req}^{a_i}$, the confirmation message sent by  $a_i$ to $a_j$, based on $\mathbf{Reply}^{a_j} \quad \forall a_j \in \mathcal{N}$, has the following structure:
\begin{equation*}
          \mathbf{Conf}^{a_i}_{a_j}=\{(\sigma_d, \pi_d, c^{a_j}_{d},  T^{a_j}_d)\ \forall (\sigma_d, \pi_d, T^{a_i}_c)\in \mathbf{Req}^{a_i}\}. 
\end{equation*}
Here, $c^{a_j}_{d}$ is a Boolean variable indicating whether $a_j$ is confirmed to perform $\sigma_d$ at ROI $\pi_d$, and $T^{a_j}_d$ is the estimated start time for the collaboration. If $a_j$ is not selected, $T^{a_j}_d < 0$.
To build the confirmation messages, the Boolean variables $\{c^{a_j}_{d},\ a_j \in \mathcal{N}\}$ must satisfy two constraints: 1) Each agent in $\mathcal{N}$ can be confirmed for at most one $(\sigma_d, \pi_d) \in \mathbf{Req}^{a_i}$; 2) Exactly one agent must be confirmed for each action $(\sigma_d, \pi_d)$. Moreover, agents are selected based on their ability to perform the assisting action $\sigma_d$ as close as possible to the target time $T^{a_i}_c$.
Given $ |\mathbf{Req}^{a_i}| = M $, denote the assisting actions as $\{\sigma_d \mid d=1,\ldots, M\}$. The problem of finding $\{c^{a_j}_d\}$ can be formulated as a MIP solved for $\mathcal{R}=\mathcal{N}$
\begin{subequations}\label{eq:MIP}
    \begin{align}
        \min\limits_{\{c^{a_j}_d , a_j\in\mathcal{R}\}_{d=1}^M} &\quad \sum^{M}_{d=1}\sum^{|\mathcal{R}|}_{j=1}c^{a_j}_{d} | t^{a_j}_d-T^{a_i}_c| \label{eq:mip-cost}\\
        \text{s.t.}\ &\quad \sum^{M}_{d=1} b^{a_j}_{d}  c^{a_j}_{d}\leq1\ \ \forall a_j\in  \mathcal{R}\label{eq:mip-constr1}\\
         &\quad \sum^{ |\mathcal{R}|}_{j=1} b^{a_j}_{d} c^{a_j}_{d}=1\ \ \forall d\in \{1,...,M\}.  \label{eq:mip-constr2} 
    \end{align}
\end{subequations}
Once solved, if $c^{a_j}_d = \top$, $T^{a_j}_d = t^{a_j}_d$ otherwise, $T^{a_j}_d = -1$.
\begin{comment}
\begin{remark}\label{rmk:necessary_condition_MIP}
    A necessary condition for feasibility is $N > M$, since each action requires a distinct agent. Additionally, agent $a_i$ cannot be assigned to any actions, as its reply will have $b^{a_i}_d = \bot \ \forall \sigma_d$.
\end{remark}
\begin{remark}\label{rmk:feasibility_MIP}
    To ensure feasibility, there must exist a subset $\mathcal{N}_H \subseteq \mathcal{N}$ such that $ |\mathcal{N}_H| = M $; each agent can assist with at least one action and, there exists a combination of these agents such that we can assign to each of them exactly one action, with no overlap.
\end{remark}
\end{comment}
To handle the complexity of \eqref{eq:MIP}, which grows exponentially with the number of binary variables involved, we propose a filtering procedure to reduce the number of agents involved while maintaining optimality.
\begin{procedure2}[\textit{Filtering Procedure}]\label{proc:filtering}
    1) Define $\mathcal{N}_U\subseteq\mathcal{N}$ as the set of agents who can assist in at least one action i.e., $a_j \in \mathcal{N}_U$ if $b^{a_j}_d = \top$ for any $(\sigma_d, \pi_d, b^{a_j}_d, t^{a_j}_d) \in \mathbf{Reply}^{a_j}_{a_i}$.\\
    2) For each $\sigma_d \in \mathbf{Req}^{a_i}$, sort agents in ascending order of $\Delta^{a_j}_{d} = |t^{a_j}_d - T_c^{a_i}|$ and store the first $M$ for each $\sigma_d$ in $\mathcal{N}_F$.
\end{procedure2}
%Once filtering is complete, the MIP is solved for $\mathcal{R}=\mathcal{N}_F$.
\begin{theorem}\label{thm:MIP}
    Consider the set of agents $\mathcal{N}$ and the $M$ actions in $\mathbf{Req}^{a_i}$. Let $\mathcal{N}_F$ be the set of filtered agents given by Proc. \ref{proc:filtering}. Assume that the MIP in \eqref{eq:MIP} for $\mathcal{R}=\mathcal{N}$ is feasible. Then, the MIP in \eqref{eq:MIP} for $\mathcal{R}=\mathcal{N}$ and for $\mathcal{R}=\mathcal{N}_F$ have identical optimal solutions.
\end{theorem}
\begin{comment}
        \textbf{Feasibility}: We will show that $\mathcal{N}_F$ includes a feasible solution if $\mathcal{N}$ contains one. Let $\mathcal{N}_H \subseteq \mathcal{N}_F$ with $|\mathcal{N}_H|\geq M$. Note that by Proc. \ref{proc:filtering}, $\mathcal{N}_U$ contains agents who can assist in at least one action, ensuring each agent in $\mathcal{N}_F$ can help with at least one action. To construct $\mathcal{N}_H$, consider sets $\mathcal{N}_d$ ($d \in \{1,\ldots,M\}$), where $a_j \in \mathcal{N}_d$ if $b^{a_j}_{d}=\top$ and denote their cardinality by $|\mathcal{N}_d|=M_d$, w.l.o.g assume $M_1 \leq M_2 \leq \ldots \leq M_M$. Order agents in ascending order of $\Delta^{a_j}_d$ for all $\mathcal{N}_d$. Starting with $\mathcal{N}_{1}$, select the first agent $a_p$, assign it to $\sigma_1$, and add it to $\mathcal{N}_H$. For $\mathcal{N}_{2}$, find the first agent different from $a_p$, say $a_q$, assign it to $\sigma_2$, and add it to $\mathcal{N}_H$. Since $M_2 \geq 2$, $a_q \neq a_p$. Continue until $\mathcal{N}_M$, where $a_r$ is selected and assigned to $\sigma_M$, completing $\mathcal{N}_H$. This construction ensures that $\mathcal{N}_H \subseteq \mathcal{N}_F$ meets the properties in Rmk \ref{rmk:feasibility_MIP}, maintaining feasibility.\\ \textbf{Optimality}: 
\end{comment}
\begin{proof}
    Let $d\in\{1,...,M\}$ be the set of actions as denoted in Proc. \ref{proc:filtering} and rewrite the objective function in \eqref{eq:mip-cost} for $\mathcal{R}=\mathcal{N}$ as $\sum^{|\mathcal{N}|}_{j=1}c^{a_j}_{1}\cdot \Delta^{a_j}_1 +\ldots + \sum^{|\mathcal{N}|}_{j=1}c^{a_j}_{M}\cdot \Delta^{a_j}_M$. For each term (action), define $\mathcal{N}_d$ as the set of $M$ agents with the smallest $\Delta^{a_j}_d$ and then define $\mathcal{N}_F = \bigcup_{d=1}^{M} \mathcal{N}_d$. Rewrite the filtered objective function as $\sum_{a_j\in\mathcal{N}_1}c^{a_j}_{1}\cdot \Delta^{a_j}_1 +\ldots + \sum_{a_j\in\mathcal{N}_M}c^{a_j}_{M}\cdot \Delta^{a_j}_M$. Let $\{c^{a_j}_d\}^{*}_{\mathcal{N}}$ and $\{c^{a_j}_d\}^{*}_{\mathcal{N}_F}$ be the optimal solutions of \eqref{eq:MIP} for $\mathcal{R}=\mathcal{N}$, and $\mathcal{R}=\mathcal{N}_F$ respectively. Assuming the solutions differ, there must exist an agent $a_j \in \mathcal{N} \setminus \mathcal{N}_F$ so that $c^{a_j}_d = 1$ in $\{c^{a_j}_d\}^{*}_{\mathcal{N}}$ would result in a smaller value for the objective function compared to any agent $a_i \in \mathcal{N}_F$ where $c^{a_i}_d = 1$ in $\{c^{a_j}_d\}^{*}_{\mathcal{N}_F}$, implying $\sum^{|\mathcal{N}|}_{j=1}c^{a_j}_{d}\cdot \Delta^{a_j}_d<\sum_{a_j\in\mathcal{N}_d}c^{a_j}_{d}\cdot \Delta^{a_j}_d$.  Simplifying the inequality under the condition that $c^{a_j}_{d} = 1$ for only one $a_j$ leads to $\Delta^{a_j}_d < \Delta^{a_i}_d$. Since $a_j$ would then appear in $\mathcal{N}_d$, it must belong to $\mathcal{N}_F$ by Proc. \ref{proc:filtering} which contradicts the assumption. Hence, the optimal solutions are identical.
\end{proof}
Based on the solution of \eqref{eq:MIP} for $\mathcal{R}=\mathcal{N}_F$, $\mathbf{Conf}^{a_i}_{a_j} \forall\ a_j \in \mathcal{N}$ is built as follows: if $a_j\in\mathcal{N}_F$ and \eqref{eq:MIP} is feasible, add $(\sigma_d, \pi_d, c^{a_j}_d, T^{a_j}_d)$ to $\mathbf{Conf}^{a_i}_{a_j}$, otherwise add $(\sigma_d, \pi_d, \bot, -1)$.
Lastly, $a_i$ sends $\mathbf{Conf}^{a_i}_{a_j}$ to all $a_j\in\mathcal{N}$.
\par As in \cite{meng_paper}, the focus is on loosely coupled MAS where collaborations are sporadic compared to the total actions. This permits us to make the following assumption.
\begin{assumption}\label{ass:loosely_coupled} \textit{(Loosely Coupled System)} There exists a finite time $\mathbf{T}>0$ such that for each agent $a_i\in\mathcal{N}$ and any collaborative action $\sigma^{a_i}_c$ requested at time $t_c>0$, \eqref{eq:MIP} for $\mathcal{R}=\mathcal{N}_F$ will attain a solution within $t_c + \mathbf{T}$ for $\sigma^{a_i}_c$.
\end{assumption}
Assumption \ref{ass:loosely_coupled} allows for collaboration delays if no immediate solution is found. A plan adaptation process repeats until a solution is achieved, ensuring eventual success, as discussed in Sec. \ref{sub:res-plan-adaptation}. This assumption is reasonable for the tasks defined by \eqref{eq:recurringLTL}. We next show that based on the availability of agents and the scarcity of collaborative actions, both task progress and assistance can be accomplished. %Asm. \ref{ass:loosely_coupled} also excludes tightly coupled systems.

\subsection{Plan Adaptation}\label{sub:res-plan-adaptation}

After agent $a_i$ sends the confirmation messages to all agents $a_j \in \mathcal{N}$, two scenarios arise.
1) If \eqref{eq:MIP} for $\mathcal{R}=\mathcal{N}_F$ has no solution, $\sigma^{a_i}_c$ cannot be fulfilled based on the current replies. Agent $a_i$ will delay $\sigma^{a_i}_c$ by adding a detour of duration $T_{delay}$ (design parameter), using the self-loop $\sigma_0 = \mathit{None}$, and reattempt the request-reply-confirmation (RRC) cycle after this delay, while all the other agents proceed as planned and discard the detours. By Asm. \ref{ass:loosely_coupled}, a solution will eventually be found even if the RRC cycle might be repeated multiple times.
2) If \eqref{eq:MIP} for $\mathcal{R}=\mathcal{N}_F$ has a solution, each assisting agent $a_j$ checks if it has been selected (i.e., $c^{a_j}_d = \top$ for any $\sigma_d$). If selected, the agent updates $\tau^{a_j}$ to include $D^{a_j}(\sigma_d)$ and sets $\overline{T}^{a_j} = \bot$. If not selected, the agent continues with its original plan. The requesting agent $a_i$ sets $\overline{T}^{a_i} = \bot$.
Note that for assisting agents, we will set $\overline{T}^{a_j} = \top$ at the end of the detour to ensure they continue their plans. For $a_i$, we will set $\overline{T}^{a_i} =\top$ after $\sigma^{a_i}_c$ is completed.
\subsection{Time Synchronization}\label{subsec:res-synchro}
Suppose \eqref{eq:MIP} for $\mathcal{R}=\mathcal{N}_F$ has a solution, establishing a collaboration between the requesting agent $a_i$ and some assisting agents $a_j\in\mathcal{N}$. To successfully complete it, according to Def. \ref{def:succesful-collab} it is necessary to synchronize the execution of the actions. Let $\mathcal{S} \subseteq \mathcal{N}$ denote the set of agents involved in this collaboration. Once $a_j \in \mathcal{S} \setminus a_i$ is ready to execute $\sigma_d$, it sends a $\mathbf{Ready}^{a_j}$ message to $a_i$. When $a_i$ is ready to execute $\sigma^{a_i}_c$ and has received all $\mathbf{Ready}^{a_j}$ messages, it sends a $\mathbf{Start}$ message to all $a_j$, ensuring all agents start simultaneously. The following Lemma underpins the synchronization process.
\begin{lemma}\label{lemma:collaboration}
    There exists a finite time $T_s \geq 0$ such that, if $t$ is the time when $\mathbf{Conf}^{a_i}_{a_j}$ is sent, then by $t + T_s$ $\mathbf{Start}$ will be sent, and the agents will begin their collaboration.
\end{lemma}
\begin{proof}
    Consider an agent $a \in \mathcal{S} \setminus a_i$ in state  $\pi^{a}_{\mathcal{G}, init}$ at time $t$. After completing the assistive action $\sigma_d$ , it transitions to  $\pi^{a}_{\mathcal{G}, fin} = \langle \pi_d, \sigma_d \rangle$  via a finite sequence $\rho_d \subset \rho^a$, where  $\rho^a$ is the action sequence associated with the plan $\tau^a$. Each action $ \sigma_i \in \rho_d $ has an effective duration $ T_{\sigma_i} + d_{\sigma_i} $, where $ T_{\sigma_i} = T^{a}_{\mathcal{G}}(\sigma_i) < \infty $ and $ 0\leq d_{\sigma_i}<\infty $ is the delay (Rmk. \ref{rem:delays}).
    The time before starting $ \sigma_d $ is $ T^{a}_S = \sum_{i=0}^{|\rho_d| - 2} (T_{\sigma_i} + d_{\sigma_i}) < \infty $. After $ T^{a}_S $, $ a $ starts $ \sigma_d $ and sends $ \mathbf{Ready}^{a} $. The same holds for $ a_i $.
    Let $ T_S = \max_a \{T^{a}_S\} < \infty $ for $ a \in \mathcal{S} $. At $ t + T_S $, all agents send $ \mathbf{Ready}^{a} $, and $ a_i $ sends $\mathbf{Start}$, synchronizing the agents.
\end{proof}

\subsection{Approach Summary}
%With the following theorem, we will summarize our approach and prove its correctness.
\begin{theorem}
    Given the set of agents $\mathcal{N}$, let $a_i\in\mathcal{N}$ build the FTS $\mathcal{T}^{a_i}_{\mathcal{G}}$ \eqref{eq:agent-model} and get assigned a recurring LTL task $\varphi^{a_i}_r$ \eqref{eq:recurringLTL}. After constructing the initial plan as in Sec. \ref{subsec:res-planning} implement the  RRC cycle for any collaborative action as detailed in Alg. \ref{alg:request}, Alg. \ref{alg:reply}, and  Sec. \ref{subsec:res-confirm}.
    Then, apply the plan adaptation from Sec. \ref{sub:res-plan-adaptation}, followed by the synchronization method in Sec. \ref{subsec:res-synchro}. Under Asm. \ref{ass:loosely_coupled} this approach solves Problem \ref{problem:task}.
\end{theorem}
\begin{proof}
Consider an agent $a_i \in \mathcal{N}$. The initial plan satisfies the task specification $\varphi^{a_i}_r$, as ensured by the planner (Sec. \ref{subsec:res-planning}). 
As plan adaptation is permitted by the definition in \eqref{eq:recurringLTL}, task satisfaction is guaranteed. Furthermore, under Asm. \ref{ass:loosely_coupled} all collaborations are established within a finite time.By Lemma \ref{lemma:collaboration}, every collaboration will be successfully completed (Def. \ref{def:succesful-collab}), as there exists a finite time $T_S \geq 0$ that ensures all actions in the collaboration start simultaneously at $t + T_S$. Therefore, $\varphi^{a_i}_r$ is satisfied. Since $a_i$ can be any agent, this holds for all $a_i \in \mathcal{N}$, ensuring that the proposed approach solves Problem \ref{problem:task}.
\end{proof}
\subsection{Complexity Reduction}\label{subsec:res-complexity}
Compared to \cite{meng_paper}, by using the ROI representation instead of grid partitioning, we reduce the size of $\mathcal{T}^{a_i}_{\mathcal{M}}$, which also reduces the size of $\mathcal{T}^{a_i}_{\mathcal{G}}$. Additionally, in Alg. \ref{alg:reply}, we replace the more complex PBA $\mathcal{A_P}$ with the simpler FTS $\mathcal{T}^{a_i}_{\mathcal{G}}$ to ensure task satisfaction, leveraging the properties given by \eqref{eq:recurringLTL} to maintain LTL compliance. This reduces our approach's overall complexity, particularly since the most expensive operations in Alg. \ref{alg:reply} are the two calls to \Dijkstra, which have a quadratic cost on the number of states. Lastly, the filtering procedure before solving the MIP ensures that only $M \leq |\mathcal{N}_F| \leq M^2$ agents are involved, offering substantial complexity reduction when $N \gg M$. Moreover, by Thm. \ref{thm:MIP}, if $M=1$ Proc. \ref{proc:filtering} returns the optimal agent, solving  the MIP \eqref{eq:MIP} directly under negligible computational cost.

% !TEX root = template.tex

\section{Experimental Results}\label{sec:experimental}
We conducted experiments, within the framework of the CANOPIES project \cite{canopies}, to show the advantages of our approach, most of which are based on the following system.
\subsection{System description}\label{subsec:exp-system}
We consider a workspace measuring $4.2\si{m} \times 5.2\si{m}$ (Fig. \ref{fig:workspace}), and two types of agents: Robotis Turtlebot3 Burger \cite{turtlebot} and Hebi Rosie with robotic arm \cite{rosie}.
\begin{figure}
    \centering
    % First column with two images stacked, labeled (a)
    \begin{minipage}{0.5\linewidth}
        \centering
        \vspace{0.14cm}
        \includegraphics[width=0.4\linewidth]{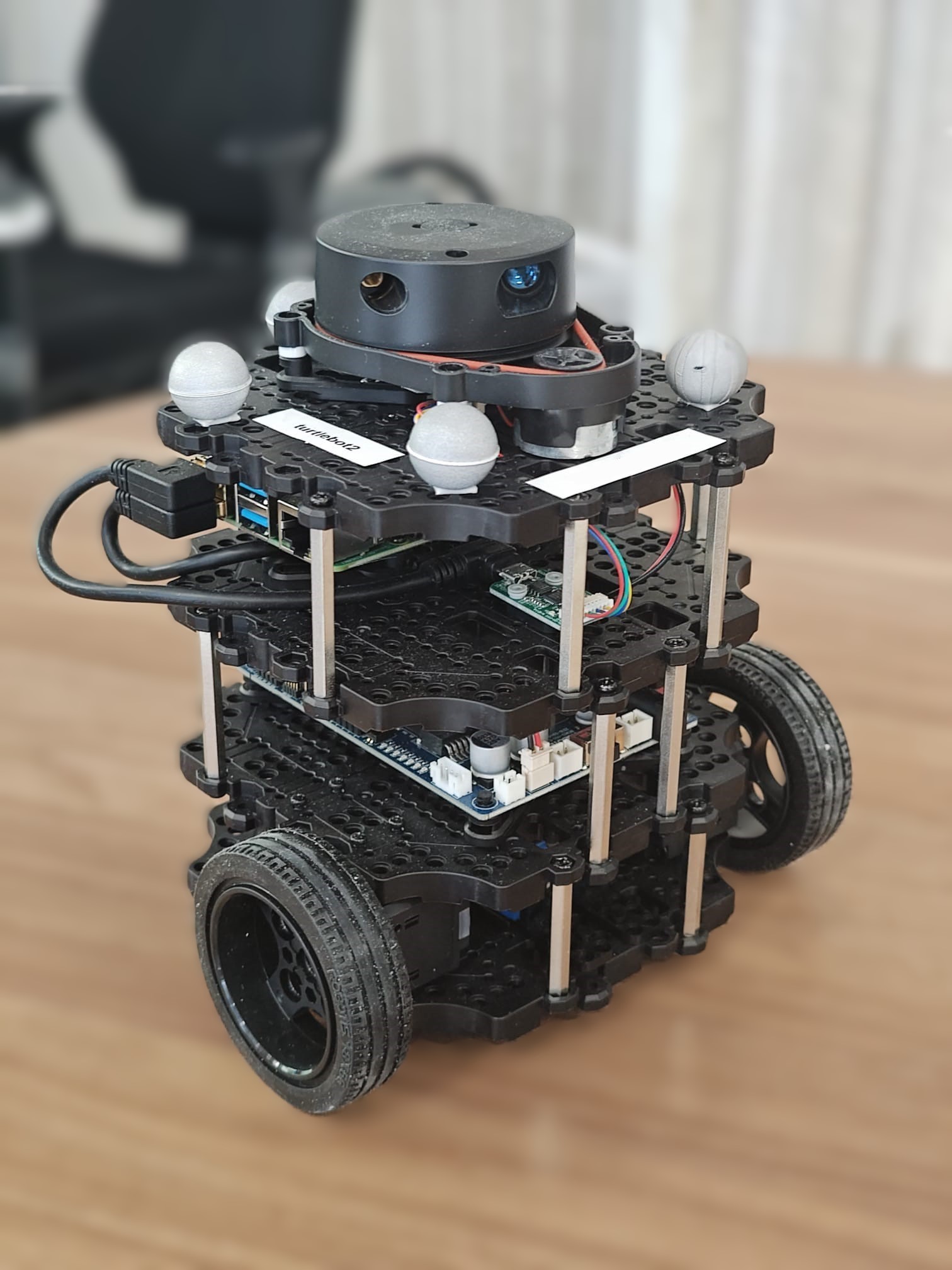}\\
        \vspace{0.05cm}
        \includegraphics[width=0.4\linewidth]{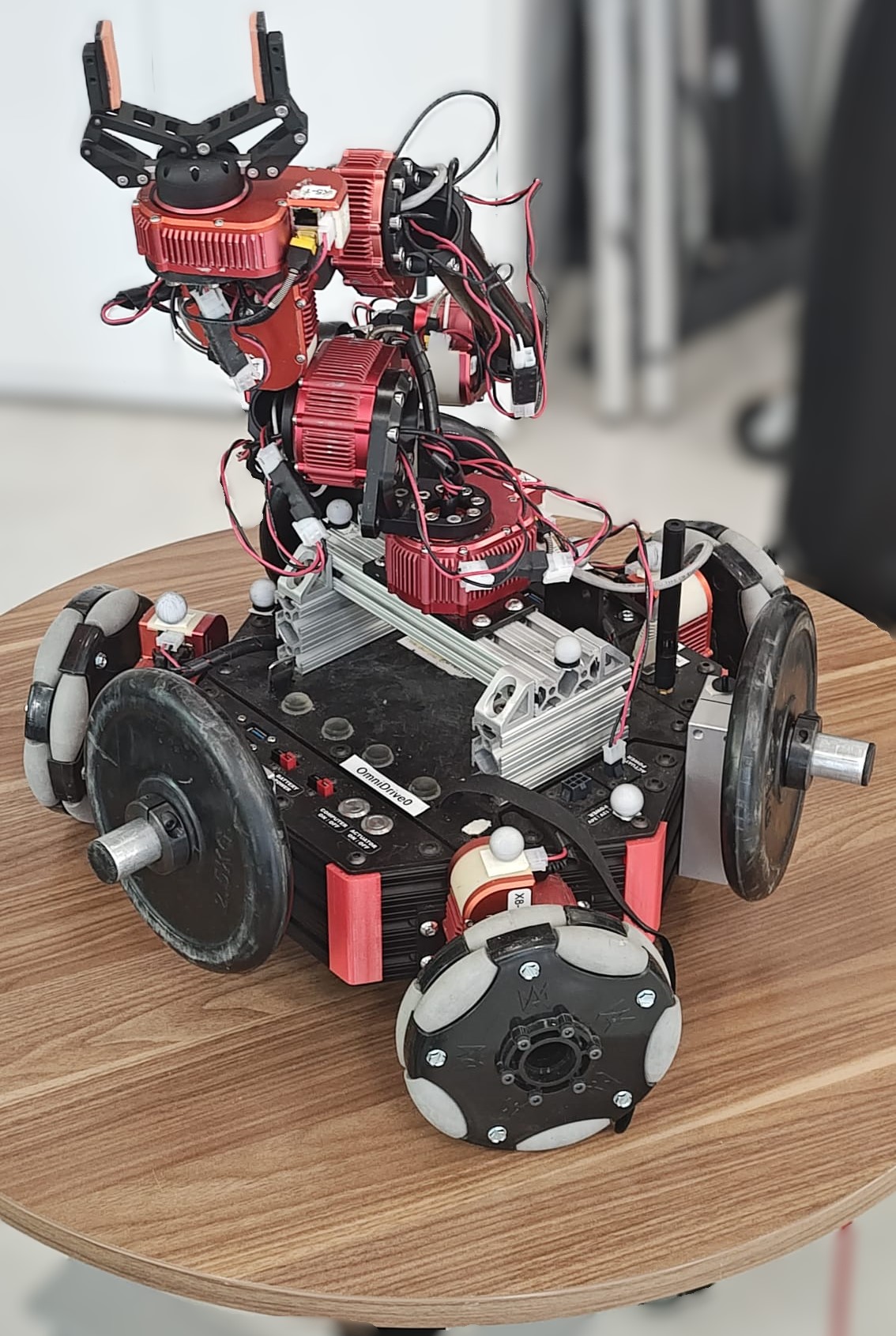}
         \subcaption{T. Turtlebot, B. Rosie}\label{fig:robots}
    \end{minipage}%
    % Second column with one image, labeled (b)
    \begin{minipage}{0.5\linewidth}
        \centering
        \includegraphics[width=\textwidth]{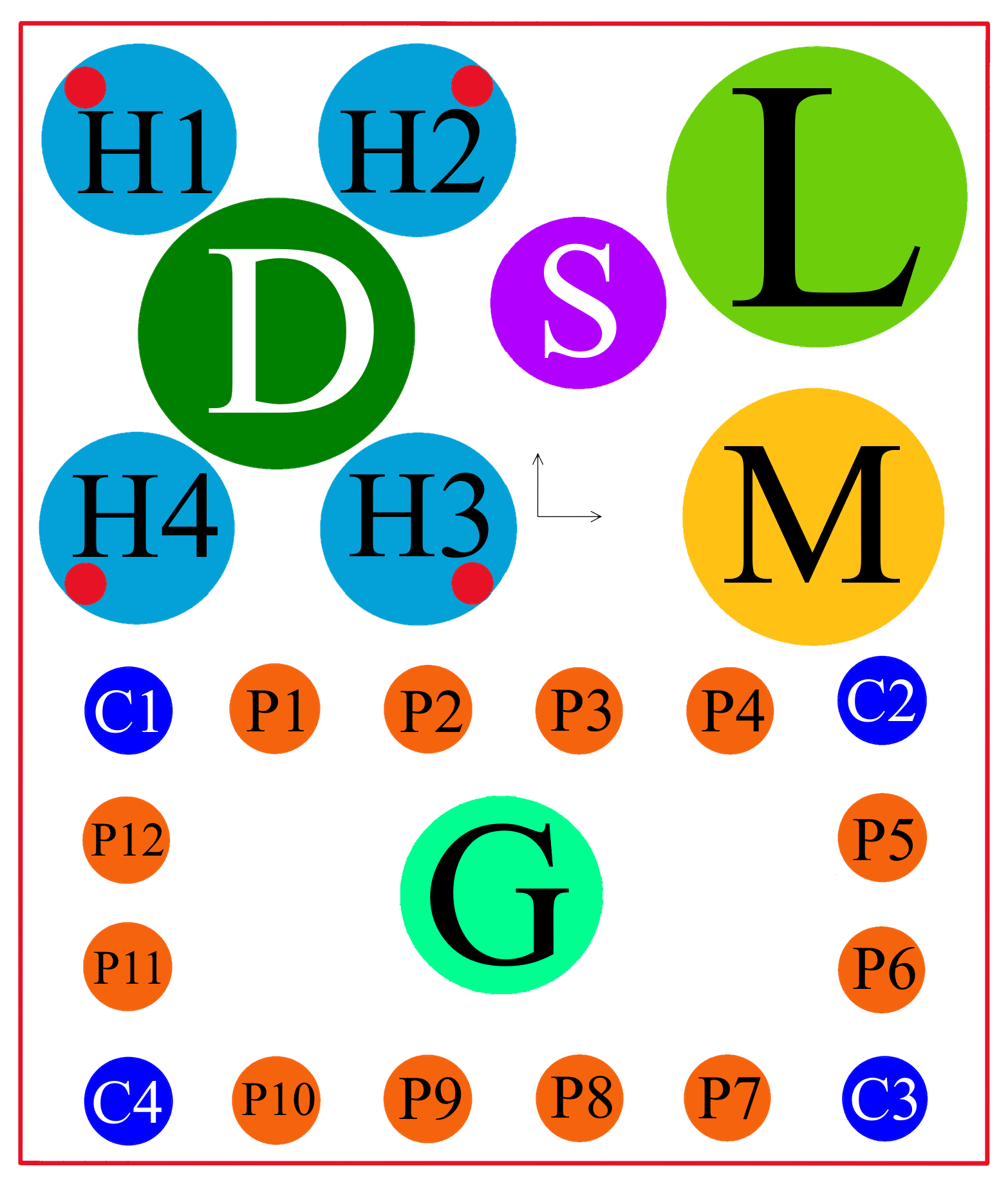}
         \subcaption{Workspace abstraction}\label{fig:workspace}
    \end{minipage}
    \caption{Experimental setup}
\end{figure}
\subsubsection{Turtlebot}
It knows the ROIs: $C1-4$, $P1-12$, $M$, and $G$. The collaborative actions are  \textit{check\_connection (cc)} at $C1$ or $C2$, \textit{group (g)} at $G$, \textit{remove\_object (ro)} at $M$, the assistive actions are \textit{help\_check\_connection (hcc)} at $C3$ or $C4$, \textit{help\_group (hg)} at $G$  and the local actions are \textit{patrol (p)} at $P1-12$, the movement related actions and, \textit{None}.
\begin{comment}
   The non-movement actions are listed in Tab. \ref{tab:act_turtlebot}.
\begin{table}[ht]
    \centering
    \begin{tabular}{c||c|c}
        \textbf{Action} &  \textbf{Type} & \textbf{Cond}\\
        \hline\hline
        \textit{check\_connection (cc)} &  collaborative & $C1$, $C2$\\
        \hline
        \textit{group (g)} &  collaborative & $G$\\
        \hline
        \textit{remove\_object (ro)} &  collaborative & $M$\\
        \hline
        \textit{help\_check\_connection (hcc)} &  assistive & $C3$, $C4$ \\
        \hline
        \textit{help\_group (hg)} &  assistive & $G$\\
        \hline
        \textit{patrol (p)} &  local & $P1-14$\\
        \hline
        \textit{None} &  local &  \\
        \hline\hline
    \end{tabular}
    \caption{Action model for Turtlebots}
    \label{tab:act_turtlebot}
\end{table} 
\end{comment}
\subsubsection{Rosie}
It knows the ROIs: $H1-4$, $D$, $L$, $M$, and $S$. The collaborative action is \textit{load (l)} at $L$, the assistive actions are \textit{help\_load (hl)} at $L$, \textit{help\_remove\_object (hro)} at $M$ and, the local actions are \textit{harvest (h)} at $H1-4$, \textit{manipulate (m)} at $M$, \textit{deliver (d)} at $D$, \textit{supervise (s)} at $S$, the movement related actions and, \textit{None}.
\begin{comment}
   \begin{table}[ht!]
    \centering
    \begin{tabular}{c||c|c}
        \textbf{Action} &  \textbf{Type} & \textbf{Cond}\\
        \hline\hline
        \textit{load (l)} &  collaborative & $L$\\
        \hline
        \textit{help\_load (hl)} &  assistive & $L$\\
        \hline
        \textit{help\_remove\_object (hro)} &  assistive & $M$\\
        \hline
        \textit{harverst (h)} &  local & $H1-4$\\
        \hline
        \textit{manipulate (m)} &  local & $M$\\
        \hline
        \textit{deliver (d)} &  local & $D$\\
        \hline
        \textit{supervise (s)} &  local & $S$\\
        \hline
        \textit{None} &  local &  \\
        \hline\hline
    \end{tabular}
    \caption{Action model for the Rosies}
    \label{tab:act_rosie}
\end{table} 
\end{comment}
\subsubsection{Task specification}\label{subsec:exp-task}
We consider a basic team composed of 3 Rosies and 6 Turtlebots. The recurring LTL tasks \eqref{eq:recurringLTL} assigned to the agents are as follows:     $\varphi^{\mathrm{rosie}_0}_r=\square\lozenge(h \wedge H1 \wedge \lozenge(h \wedge H3\wedge\lozenge d))$,    $\varphi^{\mathrm{rosie}_1}_r=\square\lozenge(s \wedge \lozenge(l \wedge\lozenge m))$,         $\varphi^{\mathrm{rosie}_2}_r=\square\lozenge(h \wedge H2 \wedge \lozenge(h \wedge H4\wedge\lozenge d))$,             $\varphi^{\mathrm{turtlebot}_0}_r=\square\lozenge(p \wedge P1 \wedge \lozenge(p\wedge P11))$,             $\varphi^{\mathrm{turtlebot}_1}_r=\square\lozenge(p \wedge P2 \wedge \lozenge(p\wedge P10))$,             $\varphi^{\mathrm{turtlebot}_2}_r=\square\lozenge(p \wedge P4 \wedge \lozenge(p \wedge P6))$,             $\varphi^{\mathrm{turtlebot}_3}_r=\square\lozenge(p \wedge P12 \wedge \lozenge(p\wedge P9\wedge \lozenge(cc \wedge C1))$,      $\varphi^{\mathrm{turtlebot}_4}_r=\square\lozenge(p \wedge P3 \wedge \lozenge(p\wedge P8\wedge \lozenge g))$ and, $\varphi^{\mathrm{turtlebot}_5}_r=\lozenge (ro \wedge \lozenge(p \wedge P5 \wedge \lozenge(cc \wedge C2))) \wedge\square\lozenge(p \wedge P5 \wedge \lozenge(p\wedge P7))$. The starting ROIs are respectively $H1$, $M$, $H2$, $P1$, $P2$, $P4$, $P12$, $P3$, $M$.
Lastly, we define \ChooseROI, which is relevant only when $check\_connection$ must be completed. If $cc$ is completed in $C1$, then $hcc$ is completed in $C3$. If $cc$ is completed in $C2$, then $hcc$ is completed in $C4$. 
\subsection{Complexity Reduction}\label{subsec:exp-complexity}
We assess the computational gains discussed in Sec. \ref{subsec:res-complexity}.
\subsubsection{Comparison between ROI and Grid Representation}
We developed an ROI representation that focuses only on the necessary regions for the agent unlike previous approaches \cite{meng_paper} which used a grid structure to partition the workspace. 
\begin{table}[b]
    \centering
    \begin{tabular}{c||c|c|c}
        \textbf{Agent} &  \textbf{States in $\mathcal{T}_{\mathcal{M}}$} &\textbf{States Grid} & \textbf{Reduction}\\
        \hline\hline
        Rosie &  $8$ & $42$ & $81.0\%$\\
        \hline
        Turtlebot&  $18$ & $500$ & $96.4\%$\\
        \hline\hline
    \end{tabular}
    \caption{Computational gains of $\mathcal{T}_{\mathcal{M}}$ over a grid structure}
    \label{tab:motion-state-reduction}
\end{table}
As shown in Tab. \ref{tab:motion-state-reduction}, the grid was partitioned with cells sized to contain the specific robot ($6\times7$ grid for Rosie and $20\times25$ for Turtlebots). 
The ROI representation led to a significant reduction in the number of states, achieving an average reduction of $88.7\%$ compared to the grid representation. This reduction influences the size of the FTS $\mathcal{T}_{\mathcal{G}}$ and the PBA $\mathcal{A_P}$. In the sequel, we focus on the ROI representation.

\subsubsection{Gains of using the FTS}
In Alg. \ref{alg:reply}, we use the FTS $\mathcal{T}_{\mathcal{G}}$, whereas \cite{meng_paper} uses the PBA $\mathcal{A_P}$. This results in an average reduction of $84.8\%$ in the number of states across all agents. The PBA's state count increases with task complexity, hence, $turtlebot_5$, with the most complex task, benefits the most from using $\mathcal{T}_{\mathcal{G}}$ in \Dijkstra. The results are in Tab. \ref{tab:state-reduction}.
\begin{table}[t]
    \centering
    \begin{tabular}{c||c|c|c}
        \textbf{Agent} &  \textbf{States $\mathcal{T}_{\mathcal{G}}$} &\textbf{States $\mathcal{A_P}$} & \textbf{Reduction}\\
        \hline\hline
        $\mathrm{rosie}_{0,1,2}$ &  $18$ & $162$ & $88.9\%$\\
        \hline
        $\mathrm{turtlebot}_{0,1,2}$ &  $37$ & $148$ & $75.0\%$\\
        \hline
        $\mathrm{turtlebot}_{3,4}$ &  $37$ & $333$ & $88.9\%$\\
        \hline
        $\mathrm{turtlebot}_5$ &  $37$ & $592$ & $93.8\%$\\
        \hline\hline
    \end{tabular}
    \caption{Computational gains by FTS against PBA}
    \label{tab:state-reduction}
\end{table}
\subsubsection{MIP Filtering}
Lastly, we analyzed the effect of the filtering procedure on the RRC cycle and on the confirmation step. This was tested in a centralized simulation, varying the number of agents and actions in a request. CycloneDDS \cite{cyclone} was used as the ROS2 middleware for stable communication between nodes.
As shown in Fig. \ref{fig:complexity} (bottom plot), the filtering procedure significantly reduced the confirmation time, especially as the number of agents increased, while the time remained relatively constant for different numbers of actions. The greatest reduction occurs with a single requested action, where the MIP is solved by Proc. \ref{proc:filtering}, with these results being barely visible in the plot due to their low values.
However, this reduction in confirmation time has a limited impact on the overall RRC cycle, as shown in Fig. \ref{fig:complexity} (upper plot), where the bottleneck is the ROS2 communication. Significant improvements only appear with 350 agents, the maximum allowed by the workstation. This suggests that with a larger number of agents ($N \gg M$), the filtering procedure yields substantial gains.
%\vspace{-0.5cm}
\begin{figure}[t]
    \centering    \includegraphics[width=\linewidth]{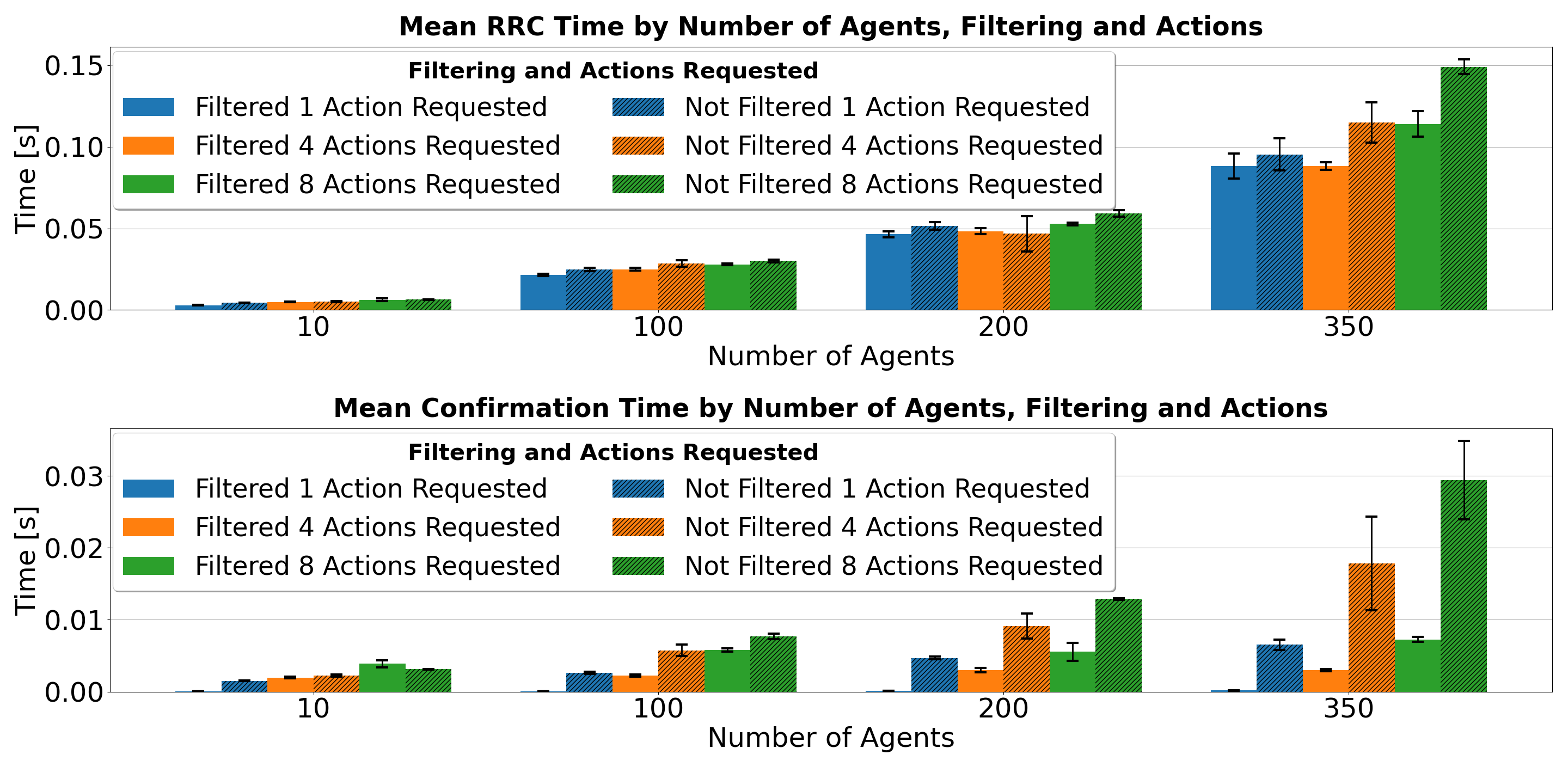}
    \caption{Effects of agents filtering}
    \label{fig:complexity}
\end{figure}

\subsection{Experimental Results}
\label{subsec:exp-results}
In this setup, we deployed our approach to the available hardware, in a decentralized way. For robot movement, we developed a model predictive controller \cite{mpc} with control barrier function \cite{cbf_1, cbf_2} constraints for collision avoidance. Agent poses were tracked using the Qualisys Motion Capture System \cite{mocap}. For the `$remove\_object$` action, we used a visual servoing controller \cite{visualservo1,visualservo2}, while all other actions were simulated. A video of the experiment is available at \cite{video}.
Fig. \ref{fig:experimental} shows the sequence of actions completed by each agent, demonstrating that all tasks were completed and that the approach synchronized collaborative and assistive actions to start simultaneously, even if some agents were ready earlier. Note that only non-movement actions are shown.
\begin{figure}[h]
    \centering    
    \includegraphics[width=\linewidth]{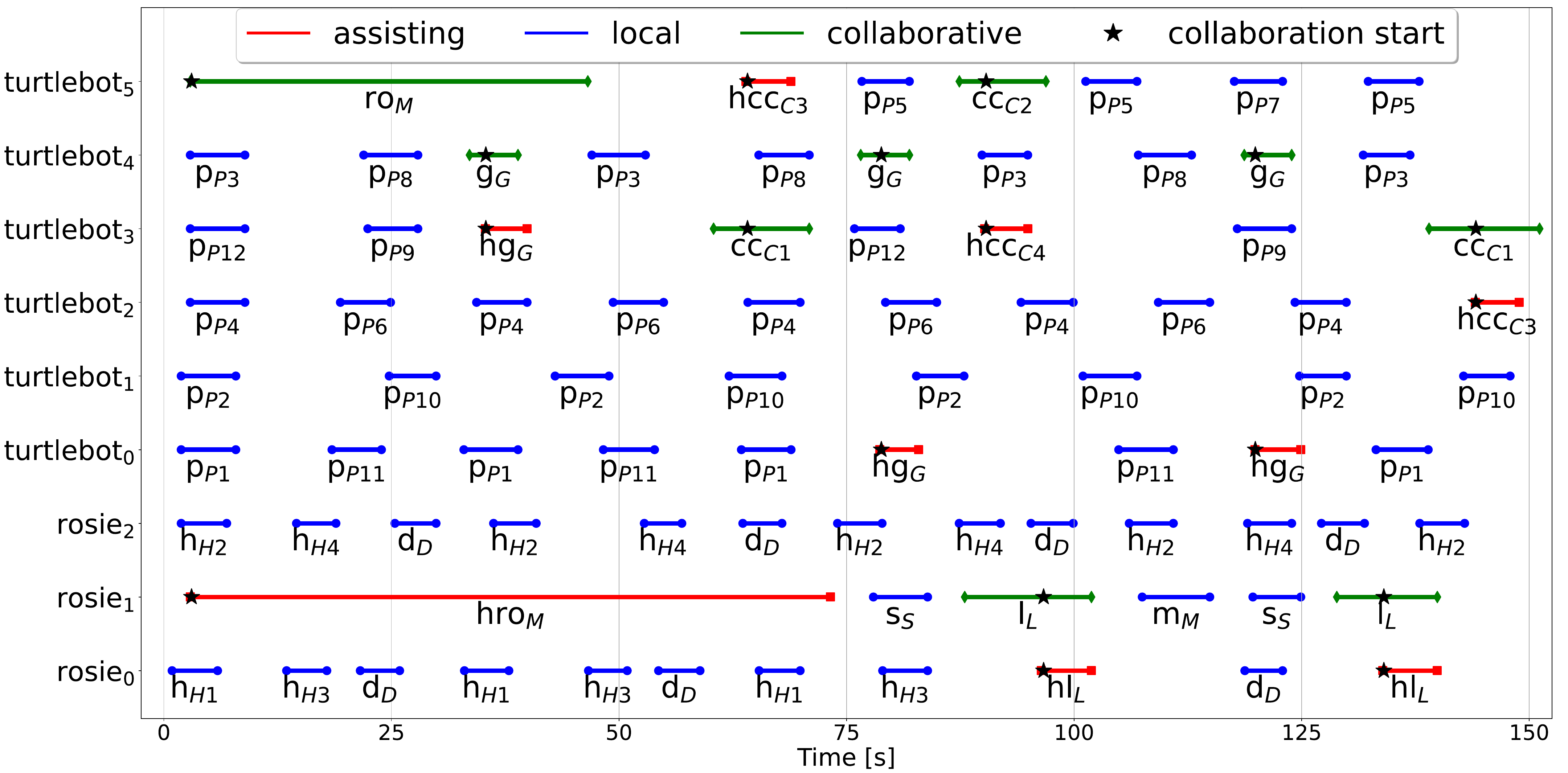}
    \caption{Agents actions in the experimental settings}
    \label{fig:experimental}
\end{figure}
\subsection{Scalability}
\label{subsec:exp-scalability}
To demonstrate the scalability of our approach, we created a simulation with 90 agents, representing 10 of the teams described in Sec. \ref{subsec:exp-system}, due to the unavailability of such a large number of robots. Fig. \ref{fig:scalability} shows the action sequence of a subset of these agents for clarity. The results indicate that the agents successfully completed their assigned tasks, collaborations, and synchronized actions. The successful completion of collaborations by the agents demonstrates the approach's effectiveness in performing well with ninety agents. Experiments with more agents could not be conducted due to RAM constraints. The results indicate that our approach scales well and can perform effectively with larger teams.
\begin{figure}[ht]
    \centering
    \includegraphics[width=\linewidth]{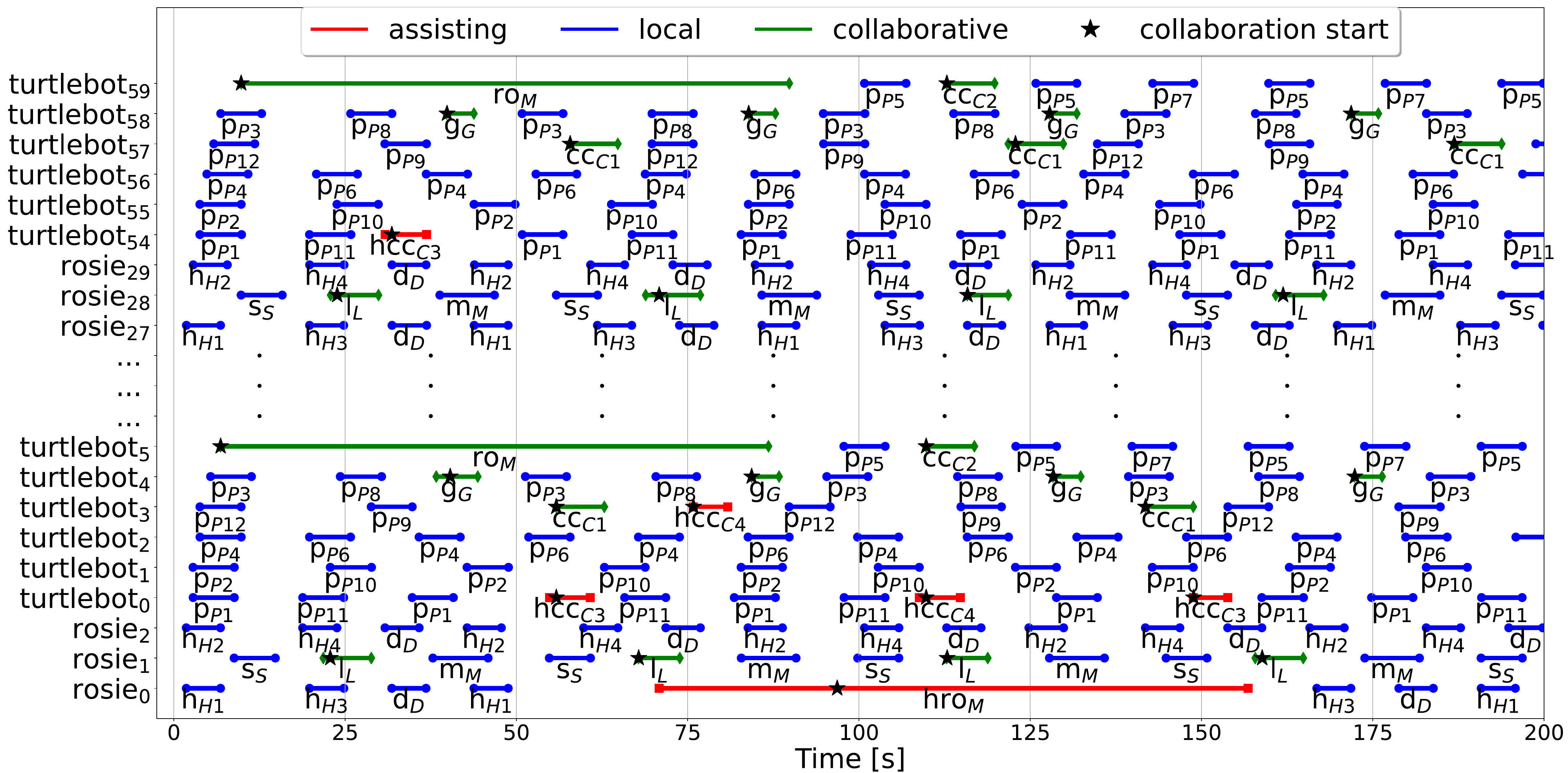}
    \caption{90 agents actions in the scalability simulation}
    \label{fig:scalability}
\end{figure}
% \vspace{-0.12cm}

% !TEX root = template.tex

\section{Conclusion}
\label{sec:conclusion}
This work focuses on MAS coordination and synchronization under recurring LTL. We extended the bottom-up scheme for distributed motion and task coordination of MAS in \cite{meng_paper}, reducing computational complexity to enhance scalability and enable deployment on robotic hardware. The package was developed in ROS2, with a synchronization mechanism to handle action delays in experiments. Future work will focus on developing additional actions and incorporating human-in-the-loop scenarios.

\newpage

\bibliography{biblio}
\bibliographystyle{IEEEtran}

\end{document}